\documentclass[11pt]{article}

\usepackage[preprint]{acl}

\usepackage{times}
\usepackage{latexsym}

\usepackage[T1]{fontenc}

\usepackage[utf8]{inputenc}

\usepackage{microtype}

\usepackage{inconsolata}

\usepackage[dvipsnames]{xcolor}
\usepackage{graphicx}
\usepackage[most]{tcolorbox}

\usepackage{amsmath}
\usepackage{bm}
\usepackage{bbm}
\usepackage{amssymb}
\usepackage{amsthm}
\usepackage{stmaryrd}

\usepackage{booktabs}
\usepackage{colortbl}
\newcolumntype{a}{>{\columncolor{Gray!20}}c}

\usepackage{tikz}
\usepackage{pgfplots}
\pgfplotsset{compat=1.18}
\usetikzlibrary{math,calc,positioning,intersections,fit,matrix,datavisualization}
\usetikzlibrary{decorations.pathreplacing,decorations.markings,arrows.meta}
\usetikzlibrary{shapes.geometric}
\usetikzlibrary{pgfplots.groupplots}
\tikzset{>=latex}
\usetikzlibrary{shapes.misc}

\usepackage[noend]{algpseudocode}

\usepackage{amsmath,amsfonts,bm}

\def\eqref#1{equation~\ref{#1}}

\def\1{\bm{1}}

\def\rr{{\textnormal{r}}}

\def\rx{{\textnormal{x}}}

\def\rvx{{\mathbf{x}}}

\def\vp{{\bm{p}}}
\def\vq{{\bm{q}}}

\def\vs{{\bm{s}}}

\def\evp{{p}}
\def\evq{{q}}

\def\evs{{s}}

\DeclareMathAlphabet{\mathsfit}{\encodingdefault}{\sfdefault}{m}{sl}
\SetMathAlphabet{\mathsfit}{bold}{\encodingdefault}{\sfdefault}{bx}{n}

\newtheorem{proposition}{Proposition}

\title{On the Rejection Criterion for Proxy-based Test-time Alignment}

\author{
    Ayoub Hammal\textsuperscript{1} \qquad Pierre Zweigenbaum\textsuperscript{1} \qquad Caio Corro\textsuperscript{2} 
    \\
    \textsuperscript{1}Université Paris-Saclay, CNRS, LISN
    \\
    \textsuperscript{2}INSA Rennes, IRISA, CNRS, Université de Rennes
    \\
    \texttt{\{ayoub.hammal,pz\}@lisn.fr} \qquad\texttt{caio.corro@irisa.fr}
}

\begin{document}
\maketitle
\begin{abstract}
Recent work proposed test-time alignment methods that rely on a small aligned model as a proxy that guides the generation of a larger base (unaligned) model.
The implicit reward approach skews the large model's distribution \cite{mitchell2024emulator},
whereas the nudging approach defers the generation of the next token to the small aligned model when the large base model is unconfident about its outcome \cite{fei2025nudging,hammal-etal-2026-kad}.
In this work, we first show that both approaches can be reduced to sampling from similar graphical models, differing only in the definition of a rejection criterion (or distribution).
Moreover, we argue that the confidence criterion is ill-motivated due to linguistic phenomena like ambiguous phrasing.
We propose a novel rejection criterion based on a \emph{conservative confidence bet}.
Experimentally, our novel approach outperforms previous work on several datasets.
\end{abstract}

\section{Introduction}

Recently, large language models (LLMs) have been the main focus of the natural language processing research community.
Training an LLM is a costly and difficult process, which can be summarized into the following three phases:
\begin{enumerate}
    \item \emph{pretraining}, where the language model is trained on a large collection of texts;
    \item \emph{supervised fine-tuning} (SFT), where the model is trained on selected input/output pairs, for example to learn to follow instructions;
    \item \emph{preference learning}, which encompasses a diverse set of extra learning steps including reinforcement learning from human feedback (RLHF), direct preference optimization (DPO), and reinforcement learning from verifiable rewards (RLVR).
\end{enumerate}
Phases~2 and~3, which are often referred to as \emph{alignment} phases, are essential to bias LLMs so that they follow expected (human) preferences, formatting or instructions, and more generally any downstream task requirements.
Unfortunately, this procedure is costly and scales poorly with the model size \cite{lambert2025tulu}.

As such, test-time alignment has been proposed as an alternative to the alignment phases.
In this setting, the distribution of a base model is biased directly during the generation process.
A first approach is to rely on an explicit reward.
If the reward is partial (or unidirectional, \emph{i.e.}, it can be applied on unfinished generations), it is then possible to directly select next token candidates by using the LLM and the reward in a pipeline \cite[][\emph{inter alia}]{deng2023rewardaugmenteddec,cao2023systematic,khanov2024args}.
However, depending on the reward architecture, partial rewards can be costly to use as they require one forward pass per next token candidate \cite{deng2023rewardaugmenteddec} or are of limited expressivity \cite{troshin2025lowrank}, and they are difficult to train \cite{rashid2025a}.
Otherwise, for rewards defined on finished generations only,
previous work proposed to rely on generation reranking \cite{gao2023scalinglaws,lietal2023making} and Markov-chain Monte Carlo sampling \cite{faria2025sampledontsearchrethinking}.
An important downside of these methods is that they require drawing many samples from the base LLM, which means the generation process is very slow.

A second approach is to bias a large base LLM with a small aligned LLM, which we call proxy-based test-time alignment in the following.
A popular approach is to extract an \emph{implicit partial reward} from the small aligned model,
and use it to ``emulate'' the distribution of the (unknown) large aligned LLM from its base version \cite{mitchell2024emulator,liu2024tuningbyproxy,li2025rewardshifted}.
However, building the implicit reward requires access to both the base and aligned versions of the small model.
Recently, \citet{fei2025nudging} proposed the \emph{nudging} method: at each generation step,
they either generate the next token from the base large model if it has a low min-entropy, or otherwise from the aligned small model.
\citet{hammal-etal-2026-kad} extended nudging to token-specific decision, \emph{i.e.}, they propose a framework in which the decision to defer generation to the aligned small model depends on the token sampled from the large one.
However, their approach is based on the same intuition: the generation should be deferred to the aligned model when the base model is \emph{unconfident about its outcome}.

\paragraph{Contributions.}
In this work, we introduce a new graphical model for proxy-based test-time alignment that can be parametrized through a \emph{rejection criterion} (or distribution).
We show that both implicit reward and nudging approaches can be modeled in this framework, showing that they are related, although very different at first glance.
Moreover, we argue that the confidence criterion of nudging is ill-motivated due to linguistic phenomena like ambiguous phrasing.
We propose instead a novel rejection criterion based on a \emph{conservative confidence bet}.
Experimentally, our novel approach outperforms previous work on several datasets.
Code to reproduce experiments is publicly available.\footnote{\url{https://github.com/ayoubhammal/knapsack-approximation-deferral}}

\paragraph{Notations.}
We write $p$ and $q$ the large and small language models, respectively.
We will distinguish the aligned version of each model as $p^\ast$ and $q^\ast$.
The token vocabulary is the set $V = \{1, \dots, |V|\}$.
To simplify notation, we drop dependence on the generation context, and we note the next token probability vector with bold letters such as $\vp$, indexed by the vocabulary, that is the probability of token $v \in V$ is written $\evp_v$.
Note that $\vp \in \triangle(|V|)$ where $\triangle(k)$ is the $(k\!-\!1)$-dimensional simplex.
Finally, we write Iverson brackets as $\llbracket \cdot \rrbracket$, which evaluates to 1 if the inner condition is true and 0 otherwise.

\section{Rejection-based Test-time Alignment}

\begin{figure}
    \centering
    \begin{tikzpicture}[
  node/.style = {
    draw,
    circle,
    inner sep=0pt,
    minimum size=20pt
    },
  >=Stealth, 
  every edge/.style={draw, thick}
  ]

  \node[node] (xprime) {$\bar\rx$};
  \node[node, right=10mm of xprime] (a) {$\rr$};
  \node[node, fill=gray!30, right=10mm of a] (x) {$\rx$};

  \draw[->] (xprime) edge (a);
  \draw[->] (a) edge (x);
  \path[->] (xprime) edge[bend left=30] (x);
\end{tikzpicture}
\caption{Probabilistic graphical model of the distribution $\pi$ for proxy-based test-time alignment.}
\label{fig:pgm}
\end{figure}
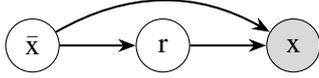

In this section,
we show that previous work on proxy-based test-time alignment can be modeled as sampling from a \emph{probabilistic graphical model} (PGM),
where a rejection criterion decides whether a sample from the base large model $\vp$ should be used as the next token or not.

\paragraph{Graphical model.}
Let $\pi$ be a PGM on the following random variables (RV):
\begin{itemize}
    \item $\bar\rx$, a latent RV on the vocabulary;
    \item $\rr$, a latent binary RV whose samples indicate whether the $\bar\rx$ sample should be rejected ($\rr = 1$) or not ($\rr = 0$);
    \item $\rx$, the observed RV over the vocabulary  (\emph{i.e.}, the generated token) .
\end{itemize}
The generative story is defined as follows:
\begin{enumerate}
    \item Sample a draft token $w$ from $\pi(\bar\rx)$;
    \item Sample a decision from the rejection distribution given the token $w$ sampled before, \emph{i.e.}, $r \sim \pi(\rr~|~\bar\rx = w)$;
    \item Sample a final token $v \sim \pi(\rx~|~\rr = r, \bar\rx = w)$,
    where we either copy the draft token if $\rr = 0$ or sample a new token otherwise.
\end{enumerate}
This PGM is depicted in Figure \ref{fig:pgm}, and the resulting distribution has the following form:
\begin{align*}
    \pi(\rx \!=\! v) =
    & \sum_{w \in V} \pi(\bar\rx\!\!=\!\!w) \begin{array}[t]{l}
        \pi(\rr\!=\!0~|~\bar\rx\!=\!w) \\
        \null\hspace{-0.6cm}\times  \pi(\rx\!=\!v~|~\rr\!=\!0,\bar\rx\!=\!w)
    \end{array}
    \nonumber \\
    & +\hspace{-0.2em}\sum_{w \in V} \pi(\bar\rx\!=\!w)
    \begin{array}[t]{l}
        \pi(\rr\!=\!1~|~\bar\rx\!=\!w)\\
        \null\hspace{-0.6cm}\times \pi(\rx\!=\!v~|~\rr\!=\!1,\bar\rx\!=\!w).
    \end{array}
\end{align*}
The latent token distribution is set to the large base model distribution:
\begin{align*}
    \pi(\bar\rvx = w) &= \evp_w,
\intertext{and we sample from the aligned small model if the draft sample is rejected, that is:}
    \pi(\rx = v~|~\rr = r,\bar\rx = w) &= \begin{cases}
        \llbracket v = w \rrbracket& \null \hspace{-0.4em}\text{if~} r = 0,\\
        \evq^\ast_v& \null \hspace{-0.4em}\text{otherwise.}
    \end{cases}
\end{align*}
As such, we are left with only defining the rejection distribution $\pi(\rr~|~\bar\rvx)$, that is, defining the Bernoulli parameters $\mu_v \in [0, 1]$:
\[
\pi(\rr = 1~|~\bar\rvx = v) = \mu_v.
\]

\paragraph{Nudging and dual KAD.}
On the one hand, the nudging method of \citet{fei2025nudging} corresponds to using the following rejection distribution:
\[
    \pi(\rr=1~|~\bar\rx=v) = \llbracket \max_w~\evp_w < \lambda \rrbracket,
\]
where $\lambda \in [0, 1]$ is a thresholding hyperparameter on the distribution confidence.
Interestingly, this rejection criterion does not depend on the initial draft token $v$, that is:
\[
    \pi(\rr=1~|~\bar\rx=v) = \pi(\rr=1).
\]
This means the rejection decision is a distribution-level decision in this case, \emph{i.e.}, we reject the sampling distribution $\vp$, not the exact sampled token.

On the other hand, the dual version of knapsack approximate deferral \citep[KAD,][]{hammal-etal-2026-kad} corresponds to using the following distribution:
\[
    \pi(\rr=1~|~\bar\rx=v) = \llbracket \evp_v < \lambda \rrbracket.
\]
In this case, the token $v$ is rejected if $\vp$ is not confident enough in the probability of $v$.

Note, however, that both of these rejection criteria are deterministic, that is:
\[
\pi(\rr = 1~|~\bar\rx=v) \in \{0, 1\},\quad \forall v \in V.
\]

\paragraph{Implicit reward.}
Implicit reward alignment \citep{mitchell2024emulator,liu2024tuningbyproxy,li2025rewardshifted} relies on extracting a token-level reward component as the ratio $q^\ast / q$, which represents the knowledge learned by $q^\ast$ during alignment.
This partial reward is then used to bias $p$ with the extracted alignment information, constructing a new aligned distribution $\vs$ defined as follows:
\[
    \evs_v = \evp_v \frac{\evq^\ast_v}{\evq_v Z},
    \quad
    \text{with}
    \quad
    Z = \sum_w \evp_w \frac{\evq^\ast_w}{\evq_w},
\]
where $Z$ is for normalization, \emph{i.e.}, to ensure that $\vs \in \triangle(|V|)$.
Although this approach seems different from the nudging and KAD,
the following proposition shows it can also be reduced to defining a rejection criterion in our PGM.

\begin{proposition} \label{prop:implicit_reward_condition}
    Let $\vp$, $\vq$ and $\vq^\ast$ be probability distributions over the vocabulary such that there exists $\alpha \in [0, 1]$ with
    \[
        \forall v \in V:\quad \frac{\evs_v - \evp_v}{\evq^\ast_v} \leq \alpha \leq \frac{\evs_v}{\evq^\ast_v},
    \]
    or, alternatively,
            \begin{equation}
        \forall v \in V:\quad \evq^\ast_v \alpha \leq \evs_v \leq \evp_v + \evq^\ast_v \alpha.
            \label{eq:alpha_condition}
    \end{equation}
    Then, the rejection distribution defined as
    \[
        \forall v \in V:\quad \pi(\rr=1~|~\bar\rx=v) = \frac{\evp_v + \evq^\ast_v \alpha - \evs_v}{\evp_v},
    \]
    gives the following PGM distribution:
    \[
        \forall v \in V: \quad \pi(\bar\rx = v) = \evs_v.
    \]
\end{proposition}

\noindent
Proof is given in Appendix~\ref{app:proof_implicit_reward_condition}.
Proposition \ref{prop:implicit_reward_condition} gives a sufficient condition on $\vp$, $\vq$ and $\vq^\ast$ so that the implicit reward distribution $\vs$ can be cast in our PGM, which reduces to the ``enclosing'' constraint in Equation~\ref{eq:alpha_condition}.
Note that this proposition suggests that, in some cases, our PGM cannot represent the distribution $\vs$.

Figure \ref{fig:proxy_examples} shows distributions that either satisfy or do not satisfy the condition.

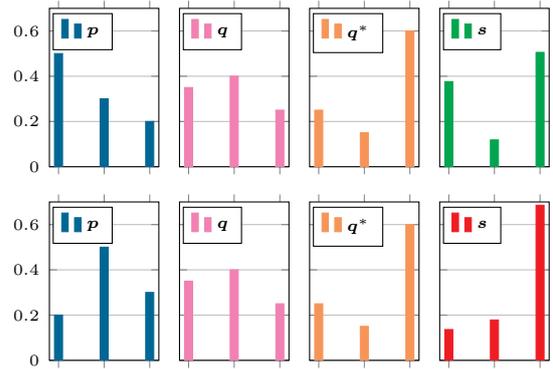
\begin{figure}
    \centering
    \begin{tikzpicture}
    \pgfplotsset{
        every axis/.style={
            ybar,
            tiny,
            yminorgrids,
            ymajorgrids,
            bar width=3.0pt,
            bar shift auto,
            x=1.2cm,
            y=3.0cm,
            xticklabels=\empty,
            legend pos=north west,
            legend columns=1,
            ymin=0,
            ymax=0.7,
            axis background/.style={fill=white},
        }
    }
    \begin{axis}[
        name=plot1,
    ]
        \addplot+[MidnightBlue] coordinates {
            (0.0, 0.5)
            (0.5, 0.3)
            (1.0, 0.2)
        };
        \addlegendentry{$\vp$};
    \end{axis}
    \begin{axis}[
        at=(plot1.right of north east),
        anchor=left of north west,
        name=plot2,
        yticklabels=\empty,
    ]
        \addplot+[CarnationPink] coordinates {
            (0.0, 0.35)
            (0.5, 0.4)
            (1.0, 0.25)
        };
        \addlegendentry{$\vq$};
    \end{axis}
    \begin{axis}[
        at=(plot2.right of north east),
        anchor=left of north west,
        name=plot3,
        yticklabels=\empty,
    ]
        \addplot+[Peach] coordinates {
            (0.0, 0.25)
            (0.5, 0.15)
            (1.0, 0.6)
        };
        \addlegendentry{$\vq^\ast$};
    \end{axis}
    \begin{axis}[
        at=(plot3.right of north east),
        anchor=left of north west,
        name=plot4,
        yticklabels=\empty,
    ]
        \addplot+[Green] coordinates {
            (0.0, 0.37608123)
            (0.5, 0.11846559)
            (1.0, 0.50545318)
        };
        \addlegendentry{$\vs$};
    \end{axis}
    
    \begin{axis}[
        at=(plot1.below south west),
        anchor=above north west,
        name=plot5,
    ]
        \addplot+[MidnightBlue] coordinates {
            (0.0, 0.2)
            (0.5, 0.5)
            (1.0, 0.3)
        };
        \addlegendentry{$\vp$};
    \end{axis}
    \begin{axis}[
        at=(plot5.right of north east),
        anchor=left of north west,
        name=plot6,
        yticklabels=\empty,
    ]
        \addplot+[CarnationPink] coordinates {
            (0.0, 0.35)
            (0.5, 0.4)
            (1.0, 0.25)
        };
        \addlegendentry{$\vq$};
    \end{axis}
    \begin{axis}[
        at=(plot6.right of north east),
        anchor=left of north west,
        name=plot7,
        yticklabels=\empty,
    ]
        \addplot+[Peach] coordinates {
            (0.0, 0.25)
            (0.5, 0.15)
            (1.0, 0.6)
        };
        \addlegendentry{$\vq^\ast$};
    \end{axis}
    \begin{axis}[
        at=(plot7.right of north east),
        anchor=left of north west,
        name=plot8,
        yticklabels=\empty,
    ]
        \addplot+[Red] coordinates {
            (0.0, 0.13600816)
            (0.5, 0.17851071)
            (1.0, 0.68548113)
        };
        \addlegendentry{$\vs$};
    \end{axis}
\end{tikzpicture}
    \caption{Two examples of implicit reward mixture distributions.
    The upper row shows an example of $\vp$, $\vq$ and $\vq^\ast$ that produce a mixture $\vs$ for which an $\alpha$ as defined in Proposition~\ref{prop:implicit_reward_condition} exists.
    The lower row shows an example of the opposite case, where no such $\alpha$ exists.
    Note that both cases differ only in distribution $\vp$, which is ``unaligned'', \emph{i.e.}, the most probable token in $\vp$ is different from the most probable token in $\vq^\ast$.
    }
    \label{fig:proxy_examples}
\end{figure}
\section{What is a Good Rejection Criterion?}

\setlength{\tabcolsep}{2.5pt}
\begin{table*}[!ht]
    \small
    \centering
    \begin{tabular}{l ccc cc a c ccc cc a}
        \toprule
        & \textbf{GMS8K} & \textbf{MATH} & \textbf{SVAMP} & \textbf{ARC} & \textbf{CSQA} & \textbf{Avg.} && \textbf{GMS8K} & \textbf{MATH} & \textbf{SVAMP} & \textbf{ARC} & \textbf{CSQA} & \textbf{Avg.} \\
        \cmidrule(lr){2-7}  \cmidrule(lr){9-14}
        & \multicolumn{6}{c}{OLMo 2 (1B and 13B)} & & \multicolumn{6}{c}{Qwen 3 (1.7B and 14B)} \\
        \midrule
        \multicolumn{14}{l}{\textbf{Building blocks}}\\
        \midrule
        $\vp$               & $54.5$ & $09.4$ & $57.6$ & $29.6$ & $19.4$ & $34.1$ && $75.5$ & $51.8$ & $80.0$ & $86.6$ & $76.9$ & $74.1$ \\
        $\vq^\ast$          & $62.5$ & $16.4$ & $70.3$ & $43.8$ & $48.4$ & $48.2$ && $75.3$ & $53.0$ & $86.6$ & $82.9$ & $68.7$ & $73.3$ \\
        
        \midrule
        \multicolumn{14}{l}{\textbf{Previous works}}\\
        \midrule

        Imp.\ reward        & $58.4$ & $18.2$ & $73.0$ & $\mathbf{63.3}$ & $\underline{55.8}$ & $53.7$ && $80.7$ & $60.6$ & $89.0$ & $88.9$ & $78.1$ & $79.4$ \\

        \multicolumn{4}{l}{$\hookrightarrow$ Nudging \cite{fei2025nudging}}&&&&&&&&&&\\[0.1em]
        $\lambda\!=\!0.3$   & $60.3$ & $13.8$ & $67.3$ & $53.7$ & $33.4$ & $45.7$ && $78.1$ & $52.8$ & $85.0$ & $88.8$ & $78.5$ & $76.6$ \\
        $\lambda\!=\!0.4$   & $61.9$ & $18.8$ & $66.6$ & $57.8$ & $46.3$ & $50.2$ && $79.5$ & $57.8$ & $87.3$ & $89.6$ & $79.3$ & $78.7$ \\[0.3em]

        \multicolumn{4}{l}{$\hookrightarrow$ Dual KAD \cite{hammal-etal-2026-kad}}&&&&&&&&&&\\[0.1em]
        $\lambda = 0.3$     & $69.5$ & $22.8$ & $76.6$ & $61.6$ & $52.3$ & $56.5$ && $81.3$ & $59.6$ & $89.0$ & $\underline{91.4}$ & $\mathbf{81.3}$ & $\mathbf{80.5}$ \\
        $\lambda = 0.4$     & $\mathbf{72.3}$ & $23.4$ & $75.3$ & $61.9$ & $55.6$ & $57.7$ && $81.7$ & $60.6$ & $87.3$ & $\mathbf{91.5}$ & $\underline{80.7}$ & $\underline{80.3}$ \\
        
        \midrule
        \multicolumn{14}{l}{\textbf{Our work: confidence bet}}\\
        \midrule

        $\lambda = 0$       & $71.1$ & $21.4$ & $77.6$ & $56.6$ & $53.4$ & $56.0$ && $78.3$ & $\mathbf{61.8}$ & $\mathbf{90.3}$ & $87.3$ & $74.7$ & $78.4$ \\
        $\lambda = 0.1$     & $70.5$ & $\underline{24.0}$ & $\mathbf{79.3}$ & $59.5$ & $\mathbf{58.7}$ & $\underline{58.4}$ && $\mathbf{82.4}$ & $59.8$ & $89.0$ & $89.5$ & $79.2$ & $79.9$ \\
        $\lambda = 0.2$     & $\underline{71.7}$ & $\mathbf{26.4}$ & $\underline{79.0}$ & $\underline{62.6}$ & $54.9$ & $\mathbf{58.9}$ && $\underline{82.1}$ & $\underline{61.6}$ & $\underline{89.3}$ & $90.5$ & $79.3$ & $\mathbf{80.5}$ \\

        \midrule
        \multicolumn{14}{l}{\textbf{For reference}}\\
        \midrule
        $\vp^\ast$          & $84.3$ & $39.6$ & $87.6$ & $82.5$ & $76.9$ & $74.1$ && $82.4$ & $64.0$ & $88.3$ & $93.8$ & $83.1$ & $82.3$ \\
        $\vq$               & $14.4$ & $01.0$ & $17.6$ & $19.6$ & $15.8$ & $13.6$ && $21.9$ & $12.2$ & $25.3$ & $47.1$ & $29.5$ & $27.2$ \\

        \bottomrule
    \end{tabular}
    \caption{Accuracies given by different models and methods. Best results on each column are in bold, second best are underlined.}
    \label{tab:accuracy}
\end{table*}

Nudging and KAD rely only on $p$ to take the rejection decision.
This is based on the assumption that the fallback distribution is an expert distribution whose error can be ignored \citep{chow1970errorreject}.
However, in practice, $q^\ast$ can underperform $p$.
Considering information from $q^\ast$ can yield a richer rejection decision.

Moreover, relying on the absolute confidence of $p$ alone is a bad criterion to assess the quality of a generation \citep{dinh2025generativeunderconfident}.
In many cases, probability mass can be split across several \emph{correct} tokens, which reduces the overall probability assigned to each of them.
Although this appears to indicate the model is unconfident, it actually shows ambiguous phrasing and related linguistic phenomena.
A naive example can be seen in splitting the model probability mass between ``\texttt{like}'' and ``\texttt{such}'', which are both valid introductions of examples (``\texttt{frameworks like Pytorch}'' vs ``\texttt{frameworks such as Pytorch}''). As both those continuations are valid, a uniform probability split would result in $\leq 0.5$ probability mass for each of these options, \emph{e.g.}, $p_\mathrm{that} = 0.35, p_\mathrm{like} = 0.35$ and the rest of the mass is spread on other tokens, all of them having very low probabilities.
This case corresponds to natural ambiguity in the language, and not to a lack of confidence of the model \emph{per se}.
As such, we argue that confidence-based rejection criteria are ill-motivated.

Instead of using fixed thresholds, as in nudging and KAD's dual rule, we can intuitively compare the initial draft from $p$ to the best possible guess from $q^\ast$.
In other words, we introduce a baseline to which we compare the confidence of the prediction from $p$,
and we additionally avoid rejecting the initial draft from $p$ if we have guarantees that $q^\ast$ cannot improve it.

\paragraph{Conservative confidence bet.}
A rational actor drawing an initial draft $v$ from $p$ would compare the confidence of this draft $\evp_v$ to the best confidence that $q^\ast$ can produce, that is, to $\max_w \evq^\ast_w$.
A decision to keep this draft can be taken safely if its confidence $\evp_v$ is larger than that of the best possible sample from $q^\ast$.
In the opposite case, a decision to defer is taken instead.
This results in a deferral rule that we call \emph{conservative confidence bet}, which is defined as follows:
\[
    \pi(\rr = 1~|~\bar\rvx = v) = \llbracket \evp_v < \max_w \evq^*_w - \lambda \rrbracket,
\]
where $\lambda \in [0, 1]$ is a margin hyperparameter.
Note that this rule mimics deferral rules previously proposed for cascading \cite{narasimhan2025fastercascades}.

\section{Experiments} \label{sec:experiments}

\paragraph{Datasets and models.}
We use the same experimental setup as \citet{hammal-etal-2026-kad}.
We select three \textit{math reasoning} datasets, GSM8K \citep{cobbe2021trainingverifierssolvemath}, MATH500 \citep{lightman2023letsverifystepstep} and SVAMP \citep{patel2021svamp}, and two \textit{commonsense reasoning} datasets, ARC-Challenge \citep{allenai2018arc}, CommonsenseQA \citep{talmor2019commonsenseqa} for our evaluation.\footnote{Although previous works also evaluate on TruthfulQA \citep{lin2022truthfulqa}, we did not include this dataset as the gold answer is always A in the dataset, see \url{https://huggingface.co/datasets/truthfulqa/truthful_qa}.}
Examples of prompts are given in Appendix~\ref{app:dataset_examples}.

We evaluate our deferral rule on two LLM families: OLMo~2 \citep{olmo20252olmo2furious} and Qwen~3 \citep{yang2025qwen3technicalreport}, which are open and closed source, respectively.
These families provide different model sizes and checkpoints from which we select two contrasting sizes for our experiments.

We use a temperature of $0.7$ as recommended to isolate the benefits of our method.
As in previous work, to select the margin hyperparameter $\lambda$ for our new rule, we evaluate a range of values on a small data subset.
We finally select $\{0, 0.1, 0.2\}$ for our final evaluation.

\subsection{Results Discussion}

Results are presented in Table \ref{tab:accuracy}.
Our conservative confidence bet rule produces results within the expected interval, that is, improving the performance of both building blocks $p$ and $q^\ast$ and just under the performance level of the target aligned large model $p^\ast$, and this holds across all models and tasks.
Compared to KAD's dual approximation rules, our rule with $\lambda = 0.2$ improves the mixture accuracy on the majority of tasks for OLMo~2, with large improvements in the case of the MATH500 task, reaching, on average, higher accuracy than the other baselines.

For Qwen~3, although our rule achieves higher accuracy on mathematical reasoning tasks, it falls short of KAD's dual approximation rules on commonsense reasoning tasks.
This mitigated performance improvement results in comparable average performance scores with our baselines.
This is mainly because the performance of the Qwen~3 base model on the selected tasks is close to that of its aligned version: $71.4$ compared to $80.2$, whereas OLMo~2 has a $37.4$ average accuracy-point gap between $p$ and $p^\ast$.

\section{Conclusion} \label{sec:conclusion}
In this work, we introduced a new probabilistic graphical model for the proxy-based test-time alignment problem.
This enables a better understanding of previous work and allows us to derive a new rejection criterion that leverages additional information from the proxy model to assess the confidence of the large base model.

\section*{Acknowledgments}

This work is supported by the SEMIAMOR (CE23-2023-0005) and InExtenso (ANR-23-IAS1-0004) project grants given by the French National Research Agency (ANR).
This work was granted access to the HPC resources of IDRIS under the allocation 2024-AD011015801 made by GENCI.

\section*{Limitations}
As per previous work, hyperparameter selection requires tuning on a separate development set.
We show however that even the absence of our margin hyperparameter ($\mu = 0$) can lead to very competitive performances.

\bibliography{custom}

@inproceedings{khanov2024args,
  title         = {{ARGS}: Alignment as Reward-Guided Search},
  author        = {Maxim Khanov and Jirayu Burapacheep and Yixuan Li},
  booktitle     = {The Twelfth International Conference on Learning Representations},
  year          = {2024},
  url           = {https://openreview.net/forum?id=shgx0eqdw6}
}

@inproceedings{deng2023rewardaugmenteddec,
  title         = "Reward-Augmented Decoding: Efficient Controlled Text Generation With a Unidirectional Reward Model",
  author        = "Deng, Haikang  and Raffel, Colin",
  editor        = "Bouamor, Houda  and Pino, Juan  and Bali, Kalika",
  booktitle     = "Proceedings of the 2023 Conference on Empirical Methods in Natural Language Processing",
  month         = dec,
  year          = "2023",
  address       = "Singapore",
  publisher     = "Association for Computational Linguistics",
  url           = "https://aclanthology.org/2024.emnlp-main.721/",
  doi           = "10.18653/v1/2023.emnlp-main.721",
  pages         = "11781--11791"
}

@article{troshin2025lowrank,
title={On the Low-Rank Parametrization of Reward Models for Controlled Language Generation},
author={Sergey Troshin and Vlad Niculae and Antske Fokkens},
journal={Transactions on Machine Learning Research},
issn={2835-8856},
year={2025},
url={https://openreview.net/forum?id=cjRsEGLT8B},
note={}
}

@inproceedings{cao2023systematic,
title={Systematic Rectification of Language Models via Dead-end Analysis},
author={Meng Cao and Mehdi Fatemi and Jackie CK Cheung and Samira Shabanian},
booktitle={The Eleventh International Conference on Learning Representations },
year={2023},
url={https://openreview.net/forum?id=k8_yVW3Wqln}
}

@InProceedings{gao2023scalinglaws,
title = 	 {Scaling Laws for Reward Model Overoptimization},
author =       {Gao, Leo and Schulman, John and Hilton, Jacob},
booktitle = 	 {Proceedings of the 40th International Conference on Machine Learning},
pages = 	 {10835--10866},
year = 	 {2023},
editor = 	 {Krause, Andreas and Brunskill, Emma and Cho, Kyunghyun and Engelhardt, Barbara and Sabato, Sivan and Scarlett, Jonathan},
volume = 	 {202},
series = 	 {Proceedings of Machine Learning Research},
month = 	 {23--29 Jul},
publisher =    {PMLR},
url = 	 {https://proceedings.mlr.press/v202/gao23h.html}
}

@inproceedings{lietal2023making,
    title = "Making Language Models Better Reasoners with Step-Aware Verifier",
    author = "Li, Yifei  and
      Lin, Zeqi  and
      Zhang, Shizhuo  and
      Fu, Qiang  and
      Chen, Bei  and
      Lou, Jian-Guang  and
      Chen, Weizhu",
    editor = "Rogers, Anna  and
      Boyd-Graber, Jordan  and
      Okazaki, Naoaki",
    booktitle = "Proceedings of the 61st Annual Meeting of the Association for Computational Linguistics (Volume 1: Long Papers)",
    month = jul,
    year = "2023",
    address = "Toronto, Canada",
    publisher = "Association for Computational Linguistics",
    url = "https://aclanthology.org/2023.acl-long.291/",
    doi = "10.18653/v1/2023.acl-long.291",
    pages = "5315--5333"
}

@misc{faria2025sampledontsearchrethinking,
  title         = {Sample, Don't Search: Rethinking Test-Time Alignment for Language Models},
  author        = {Gon\c{c}alo Faria and Noah A. Smith},
  year          = {2025},
  eprint        = {2504.03790},
  archiveprefix = {arXiv},
  primaryclass  = {cs.CL},
  url           = {https://arxiv.org/abs/2504.03790}
}

@inproceedings{mitchell2024emulator,
  title         = {An Emulator for Fine-tuning Large Language Models using Small Language Models},
  author        = {Eric Mitchell and Rafael Rafailov and Archit Sharma and Chelsea Finn and Christopher D Manning},
  booktitle     = {The Twelfth International Conference on Learning Representations},
  year          = {2024},
  url           = {https://openreview.net/forum?id=Eo7kv0sllr}
}

@inproceedings{liu2024tuningbyproxy,
  title         = {Tuning Language Models by Proxy},
  author        = {Alisa Liu and Xiaochuang Han and Yizhong Wang and Yulia Tsvetkov and Yejin Choi and Noah A. Smith},
  booktitle     = {First Conference on Language Modeling},
  year          = {2024},
  url           = {https://openreview.net/forum?id=dribhnhm1i}
}

@inproceedings{li2025rewardshifted,
    title = "Reward-Shifted Speculative Sampling Is An Efficient Test-Time Weak-to-Strong Aligner",
    author = "Li, Bolian  and
      Wu, Yanran  and
      Luo, Xinyu  and
      Zhang, Ruqi",
    editor = "Christodoulopoulos, Christos  and
      Chakraborty, Tanmoy  and
      Rose, Carolyn  and
      Peng, Violet",
    booktitle = "Proceedings of the 2025 Conference on Empirical Methods in Natural Language Processing",
    month = nov,
    year = "2025",
    address = "Suzhou, China",
    publisher = "Association for Computational Linguistics",
    url = "https://aclanthology.org/2025.emnlp-main.578/",
    doi = "10.18653/v1/2025.emnlp-main.578",
    pages = "11468--11478",
    ISBN = "979-8-89176-332-6"
}

@inproceedings{fei2025nudging,
    title = "Nudging: Inference-time Alignment of {LLM}s via Guided Decoding",
    author = "Fei, Yu  and
      Razeghi, Yasaman  and
      Singh, Sameer",
    editor = "Che, Wanxiang  and
      Nabende, Joyce  and
      Shutova, Ekaterina  and
      Pilehvar, Mohammad Taher",
    booktitle = "Proceedings of the 63rd Annual Meeting of the Association for Computational Linguistics (Volume 1: Long Papers)",
    month = jul,
    year = "2025",
    address = "Vienna, Austria",
    publisher = "Association for Computational Linguistics",
    url = "https://aclanthology.org/2025.acl-long.623/",
    doi = "10.18653/v1/2025.acl-long.623",
    pages = "12702--12739",
    ISBN = "979-8-89176-251-0"
}

@ARTICLE{chow1970errorreject,
author={Chow, C.},
journal={IEEE Transactions on Information Theory}, 
title={On optimum recognition error and reject tradeoff}, 
year={1970},
volume={16},
number={1},
pages={41-46},
doi={10.1109/TIT.1970.1054406}
}

@inproceedings{dinh2025generativeunderconfident,
    title = "Are Generative Models Underconfident? Better Quality Estimation with Boosted Model Probability",
    author = "Dinh, Tu Anh  and
      Niehues, Jan",
    editor = "Christodoulopoulos, Christos  and
      Chakraborty, Tanmoy  and
      Rose, Carolyn  and
      Peng, Violet",
    booktitle = "Proceedings of the 2025 Conference on Empirical Methods in Natural Language Processing",
    month = nov,
    year = "2025",
    address = "Suzhou, China",
    publisher = "Association for Computational Linguistics",
    url = "https://aclanthology.org/2025.emnlp-main.166/",
    doi = "10.18653/v1/2025.emnlp-main.166",
    pages = "3364--3382",
    ISBN = "979-8-89176-332-6"
}

@misc{olmo20252olmo2furious,
title={2 {OLMo} 2 Furious}, 
author={Team OLMo and Pete Walsh and Luca Soldaini and Dirk Groeneveld and Kyle Lo and Shane Arora and Akshita Bhagia and Yuling Gu and Shengyi Huang and Matt Jordan and Nathan Lambert and Dustin Schwenk and Oyvind Tafjord and Taira Anderson and David Atkinson and Faeze Brahman and Christopher Clark and Pradeep Dasigi and Nouha Dziri and Michal Guerquin and Hamish Ivison and Pang Wei Koh and Jiacheng Liu and Saumya Malik and William Merrill and Lester James V. Miranda and Jacob Morrison and Tyler Murray and Crystal Nam and Valentina Pyatkin and Aman Rangapur and Michael Schmitz and Sam Skjonsberg and David Wadden and Christopher Wilhelm and Michael Wilson and Luke Zettlemoyer and Ali Farhadi and Noah A. Smith and Hannaneh Hajishirzi},
year={2025},
eprint={2501.00656},
archivePrefix={arXiv},
primaryClass={cs.CL},
url={https://arxiv.org/abs/2501.00656}, 
}

@misc{yang2025qwen3technicalreport,
title={Qwen3 Technical Report}, 
author={An Yang and Anfeng Li and Baosong Yang and Beichen Zhang and Binyuan Hui and Bo Zheng and Bowen Yu and Chang Gao and Chengen Huang and Chenxu Lv and Chujie Zheng and Dayiheng Liu and Fan Zhou and Fei Huang and Feng Hu and Hao Ge and Haoran Wei and Huan Lin and Jialong Tang and Jian Yang and Jianhong Tu and Jianwei Zhang and Jianxin Yang and Jiaxi Yang and Jing Zhou and Jingren Zhou and Junyang Lin and Kai Dang and Keqin Bao and Kexin Yang and Le Yu and Lianghao Deng and Mei Li and Mingfeng Xue and Mingze Li and Pei Zhang and Peng Wang and Qin Zhu and Rui Men and Ruize Gao and Shixuan Liu and Shuang Luo and Tianhao Li and Tianyi Tang and Wenbiao Yin and Xingzhang Ren and Xinyu Wang and Xinyu Zhang and Xuancheng Ren and Yang Fan and Yang Su and Yichang Zhang and Yinger Zhang and Yu Wan and Yuqiong Liu and Zekun Wang and Zeyu Cui and Zhenru Zhang and Zhipeng Zhou and Zihan Qiu},
year={2025},
eprint={2505.09388},
archivePrefix={arXiv},
primaryClass={cs.CL},
url={https://arxiv.org/abs/2505.09388}, 
}

@misc{cobbe2021trainingverifierssolvemath,
title={Training Verifiers to Solve Math Word Problems}, 
author={Karl Cobbe and Vineet Kosaraju and Mohammad Bavarian and Mark Chen and Heewoo Jun and Lukasz Kaiser and Matthias Plappert and Jerry Tworek and Jacob Hilton and Reiichiro Nakano and Christopher Hesse and John Schulman},
year={2021},
eprint={2110.14168},
archivePrefix={arXiv},
primaryClass={cs.LG},
url={https://arxiv.org/abs/2110.14168}, 
}

@misc{lightman2023letsverifystepstep,
title={Let's Verify Step by Step}, 
author={Hunter Lightman and Vineet Kosaraju and Yura Burda and Harri Edwards and Bowen Baker and Teddy Lee and Jan Leike and John Schulman and Ilya Sutskever and Karl Cobbe},
year={2023},
eprint={2305.20050},
archivePrefix={arXiv},
primaryClass={cs.LG},
url={https://arxiv.org/abs/2305.20050}, 
}

@inproceedings{patel2021svamp,
title = "Are {NLP} Models really able to Solve Simple Math Word Problems?",
author = "Patel, Arkil  and
Bhattamishra, Satwik  and
Goyal, Navin",
booktitle = "Proceedings of the 2021 Conference of the North American Chapter of the Association for Computational Linguistics: Human Language Technologies",
month = jun,
year = "2021",
address = "Online",
publisher = "Association for Computational Linguistics",
url = "https://aclanthology.org/2021.naacl-main.168",
doi = "10.18653/v1/2021.naacl-main.168",
pages = "2080--2094",
}

@article{allenai2018arc,
author    = {Peter Clark  and Isaac Cowhey and Oren Etzioni and Tushar Khot and
             Ashish Sabharwal and Carissa Schoenick and Oyvind Tafjord},
title     = {Think you have Solved Question Answering? Try ARC, the AI2 Reasoning Challenge},
journal   = {arXiv:1803.05457v1},
year      = {2018},
}

@inproceedings{talmor2019commonsenseqa,
title = "{C}ommonsense{QA}: A Question Answering Challenge Targeting Commonsense Knowledge",
author = "Talmor, Alon  and
Herzig, Jonathan  and
Lourie, Nicholas  and
Berant, Jonathan",
booktitle = "Proceedings of the 2019 Conference of the North {A}merican Chapter of the Association for Computational Linguistics: Human Language Technologies, Volume 1 (Long and Short Papers)",
month = jun,
year = "2019",
address = "Minneapolis, Minnesota",
publisher = "Association for Computational Linguistics",
url = "https://aclanthology.org/N19-1421",
doi = "10.18653/v1/N19-1421",
pages = "4149--4158",
archivePrefix = "arXiv",
eprint        = "1811.00937",
primaryClass  = "cs",
}

@inproceedings{lin2022truthfulqa,
title = "{T}ruthful{QA}: Measuring How Models Mimic Human Falsehoods",
author = "Lin, Stephanie  and
Hilton, Jacob  and
Evans, Owain",
editor = "Muresan, Smaranda  and
Nakov, Preslav  and
Villavicencio, Aline",
booktitle = "Proceedings of the 60th Annual Meeting of the Association for Computational Linguistics (Volume 1: Long Papers)",
month = may,
year = "2022",
address = "Dublin, Ireland",
publisher = "Association for Computational Linguistics",
url = "https://aclanthology.org/2022.acl-long.229/",
doi = "10.18653/v1/2022.acl-long.229",
pages = "3214--3252"
}

@inproceedings{
rashid2025a,
title={A Critical Look At Tokenwise Reward-Guided Text Generation},
author={Ahmad Rashid and Ruotian Wu and Julia Grosse and Agustinus Kristiadi and Pascal Poupart},
booktitle={Second Conference on Language Modeling},
year={2025},
url={https://openreview.net/forum?id=Vnw9c1YLhV}
}

@inproceedings{
lambert2025tulu,
title={Tulu 3: Pushing Frontiers in Open Language Model Post-Training},
author={Nathan Lambert and Jacob Morrison and Valentina Pyatkin and Shengyi Huang and Hamish Ivison and Faeze Brahman and Lester James Validad Miranda and Alisa Liu and Nouha Dziri and Xinxi Lyu and Yuling Gu and Saumya Malik and Victoria Graf and Jena D. Hwang and Jiangjiang Yang and Ronan Le Bras and Oyvind Tafjord and Christopher Wilhelm and Luca Soldaini and Noah A. Smith and Yizhong Wang and Pradeep Dasigi and Hannaneh Hajishirzi},
booktitle={Second Conference on Language Modeling},
year={2025},
url={https://openreview.net/forum?id=i1uGbfHHpH}
}

@inproceedings{
narasimhan2025fastercascades,
title={Faster Cascades via Speculative Decoding},
author={Harikrishna Narasimhan and Wittawat Jitkrittum and Ankit Singh Rawat and Seungyeon Kim and Neha Gupta and Aditya Krishna Menon and Sanjiv Kumar},
booktitle={The Thirteenth International Conference on Learning Representations},
year={2025},
url={https://openreview.net/forum?id=vo9t20wsmd}
}

@inproceedings{hammal-etal-2026-kad,
    title = "{KAD}: A Framework for Proxy-based Test-time Alignment with Knapsack Approximation Deferral",
    author = "Hammal, Ayoub  and
      Zweigenbaum, Pierre  and
      Corro, Caio",
    editor = "Demberg, Vera  and
      Inui, Kentaro  and
      Marquez, Llu{\'i}s",
    booktitle = "Proceedings of the 19th Conference of the {E}uropean Chapter of the {A}ssociation for {C}omputational {L}inguistics (Volume 1: Long Papers)",
    month = mar,
    year = "2026",
    address = "Rabat, Morocco",
    publisher = "Association for Computational Linguistics",
    url = "https://aclanthology.org/2026.eacl-long.179/",
    doi = "10.18653/v1/2026.eacl-long.179",
    pages = "3854--3872",
    ISBN = "979-8-89176-380-7"
}

\clearpage
\appendix

\section{Proof of Proposition \ref{prop:implicit_reward_condition}}
\label{app:proof_implicit_reward_condition}

\begin{proof}
    Let $ \alpha \in [0, 1]$ such that the condition in Equation~\ref{eq:alpha_condition} holds.
    For all $v \in V$, we write:
    \begin{align}
        \pi(\rr = 1~|~\bar\rx=v)
        &= \frac{\evp_v + \evq^\ast_v \alpha - \evs_v}{\evp_v}
        \nonumber \\
        &= \mu_v. \label{eq:rejection_prob_implicit_reward}
    \end{align}
    We must first prove that $\pi(\rr~|~\bar\rx)$ is a valid distribution,
    that is $\mu_v \in [0, 1]$, $\forall v \in V$.

    First, we have
    \begin{align*}
        &&\frac{\evp_v + \evq^\ast_v \alpha - \evs_v}{\evp_v}
        &\leq 1
        \\
        &\iff& \evp_v + \evq^\ast_v \alpha - \evs_v
        &\leq \evp_v
        \\
        &\iff&  \alpha
        &\leq \frac{\evs_v}{\evq^\ast_v},
    \end{align*}
    which is true based on our assumptions on $\alpha$.
    Second, we have
    \begin{align*}
        && \frac{\evp_v + \evq^\ast_v \alpha - \evs_v}{\evp_v}
        &\geq 0
        \\
        &\iff&  \evp_v + \evq^\ast_v \alpha - \evs_v
        &\geq 0
        \\
        &\iff&  \alpha
        &\geq \frac{\evs_v - \evp_v}{\evq^\ast_v},
    \end{align*}
    which is also true based on our assumptions on $\alpha$.
    
    We can now compute the resulting PGM distribution. By definition, we have:
    \begin{align*}
        &\pi(\rx = v)
        = \evp_v (1 - \mu_v) + \evq^\ast_v \sum_w \evp_w \mu_w.
\intertext{
Developing the rejection probability $\mu_v$ as per Equation \ref{eq:rejection_prob_implicit_reward}, we obtain:
}
        &\quad= \evp_v \left ( 1 - \frac{\evp_v + \evq^\ast_v \alpha - \evs_v}{\evp_v} \right )
        \\
        &\qquad \qquad + \evq^\ast_v \sum_w \evp_w \frac{\evp_w + \evq^\ast_w \alpha - \evs_w}{\evp_w}
        \\
        &\quad= \evp_v - \evp_v - \evq^\ast_v \alpha + \evs_v
        \\
        &\qquad \qquad + \evq^\ast_v \left ( \sum_w \evp_w + \sum_w \evq^\ast_w \alpha - \sum_w \evs_w \right )
        \\
        &\quad= \evs_v - \evq^\ast_v \alpha + \evq^\ast_v \left ( 1 + \alpha - 1 \right )
        \\
        &\quad= \evs_v - \evq^\ast_v \alpha + \evq^\ast_v  \alpha
        \\
        &\quad= \evs_v,
    \end{align*}
    which ends the proof.
\end{proof}

\section{Dataset examples} \label{app:dataset_examples}

Following previous work, we rely on regular expressions to extract the final response form the LLM generation.
In the case of GSM8K and SVAMP, we extract the last number in the model's response.
In the case of MATH500, we extract the last \texttt{\textbackslash boxed\{\}} number in the model's response.
In the multiple-choice question tasks, we extract the last capital letter from \texttt{A} to \texttt{J} in the response.

The models are evaluated on the question prompts alone, with no additional few-shot examples.
For each task, we use the extracted answer's accuracy to evaluate the model's performance.
We also evaluate the generation throughput as the average number of tokens generated by the model each second.

The following system prompt was used in all tasks:
{
\footnotesize
\begin{verbatim}
Answer the question by walking through
the reasoning step by step.
\end{verbatim}
}

\paragraph{MATH500.}
Here is an example of a question/answer from MATH500 dataset:
{
\footnotesize
\begin{verbatim}
{
  "role": "user",
  "content": "Convert the point $(0,3)$ in
  rectangular coordinates to polar coordinates.
  Enter your answer in the  form $(r,\\theta),$
  where $r > 0$ and $0 \\le \\theta < 2 \\pi.$
  \n\nPresent the answer in LaTex format:
  \\boxed{Your answer}"
},
{
  "role": "assistant",
  "content": "We have that $r = \\sqrt{0^2
  + 3^2} = 3.$ Also, if we draw the line
  connecting the origin and $(0,3),$ this line
  makes an angle of $\\frac{\\pi}{2}$ with the
  positive $x$-axis.\n\n[asy]\nunitsize(0.8 cm);
  \n\ndraw((-0.5,0)--(3.5,0));\ndraw((0,-0.5)
  --(0,3.5)); \ndraw(arc((0,0),3,0,90),red,
  Arrow(6));\n\ndot((0,3), red);\nlabel(\"$(0,3)$
  \", (0,3), W);\ndot((3,0), red); \n[/asy]
  \n\nTherefore, the polar coordinates are
  $\\boxed{\\left( 3, \\frac{\\pi}{2} \\right)}.$"
}
\end{verbatim}
}
The following instruction is appended at the end of the user's query :
{
\footnotesize
\begin{verbatim}
\n\nPresent the answer in LaTex format:
\\boxed{{Your answer}}
\end{verbatim}
}

\paragraph{GSM8K.}
Here is an example of a question/answer from GSM8K dataset:
{
\footnotesize
\begin{verbatim}
{
  "role": "user",
  "content": "Janet\u2019s ducks lay 16 eggs per day.
  She eats three for breakfast every morning and
  bakes muffins for her friends every day with four.
  She sells the remainder at the farmers' market
  daily for $2 per fresh duck egg. How much in
  dollars does she make every day at the farmers'
  market?"
},
{
  "role": "assistant",
  "content": "How many eggs does Janet sell? ** Janet
  sells 16 - 3 - 4 = <<16-3-4=9>>9 duck eggs a day.
  \nHow much does Janet make at the farmers' market?
  ** She makes 9 * 2 = $<<9*2=18>>18 every day at the
  farmer\u2019s market.
  \n#### 18"
}
\end{verbatim}
}

\paragraph{ARC.}
Here is an example of a question/answer from ARC dataset:
{
\footnotesize
\begin{verbatim}
{
  "role": "user",
  "content": "Choose the correct answer to the
  following multiple-choice question.
  \n\nQuestion: An astronomer observes that a
  planet rotates faster after a meteorite impact.
  Which is the most likely effect of this increase
  in rotation?\n\nA). Planetary density will
  decrease.\nB). Planetary years will become longer.
  \nC). Planetary days will become shorter.
  \nD). Planetary gravity will become stronger.
  \n\nProvide your reasoning about the answer and
  finish your answer with the letter corresponding
  to the correct option (e.g., A, B, C, or D).\n\n"
},
{
  "role": "assistant",
  "content": "\nAnswer: C\n\n"
}
\end{verbatim}
}

The following prefix is prepended to the user's query
{
\footnotesize
\begin{verbatim}
Choose the correct answer to the following
multiple-choice question.\n\n
\end{verbatim}
}
The following suffix is appended to the user's query
{
\footnotesize
\begin{verbatim}
\nProvide your reasoning about the answer and finish
your answer with the letter corresponding to the
correct option (e.g., A, B, C, or D).\n\n
\end{verbatim}
}

\end{document}